\documentclass[a4paper, 11pt]{article}
\pdfoutput=1

\usepackage{amsthm}
\usepackage{amsmath}
\usepackage{amssymb}
\usepackage{amsfonts}
\usepackage{mathtools}

\usepackage[english]{babel}
\usepackage{xspace}
\usepackage[top=1in, bottom=1.25in, left=1.25in, right=1.25in]{geometry}

\usepackage{url}
\usepackage{natbib}
\usepackage[breaklinks, colorlinks, linkcolor=black, citecolor=black, 
urlcolor=black]{hyperref}
\usepackage{hypernat}

\bibpunct{(}{)}{,}{a}{,}{,}

\usepackage{placeins}
\usepackage[pdftex, dvipsnames]{xcolor}
\usepackage[caption=false]{subfig}
\usepackage{graphicx}
\graphicspath{{./images/}}

\usepackage{algorithm}
\usepackage{algpseudocode}

\def\Keywords{classifier performance, meta-learning, multiclass 
	classification, hierarchical classifier, classifier comparison}

\hypersetup{%
	pdfauthor={Gerrit J.J. van den Burg, Alfred O. Hero III},
	pdftitle={Fast Meta-Learning for Adaptive Hierarchical Classifier 
		Design},
	pdfkeywords={\Keywords},
	pdfsubject={Machine Learning}
}

\newcommand{\bft}[1]{\mathbf{#1}}
\newcommand{\bfs}[1]{\boldsymbol{#1}}

\newcommand{\abs}[1]{\left|#1\right|}
\newcommand{\norm}[1]{\left\|#1\right\|}
\newcommand{\myd}{\,\mathrm{d}}

\DeclareMathOperator*{\argmax}{arg\,max}

\newtheorem{theorem}{Theorem}
\newtheorem{lemma}[theorem]{Lemma}
\newtheorem{definition}[theorem]{Definition}

\newcommand{\nobs}{n}
\newcommand{\nvar}{d}
\newcommand{\nclass}{K}
\newcommand{\bat}{Bhattacharyya\xspace}

\title{Fast Meta-Learning for Adaptive Hierarchical Classifier Design}

\author{
	Gerrit J.J. van den Burg$^1$
	\and
	Alfred O. Hero$^2$
}

\date{
	{\small
	$^1$Econometric Institute, Erasmus University Rotterdam, The 
	Netherlands\\
	$^2$Electrical Engineering and Computer Science, University of 
	Michigan, USA\\[2ex]
	\today\\
}}

\begin{document}

\maketitle

\begin{abstract}%
	\noindent
	We propose a new splitting criterion for a meta-learning approach to 
	multiclass classifier design that adaptively merges the classes into a 
	tree-structured hierarchy of increasingly difficult binary 
	classification problems. The classification tree is constructed from 
	empirical estimates of the Henze-Penrose bounds on the pairwise Bayes 
	misclassification rates that rank the binary subproblems in terms of 
	difficulty of classification. The proposed empirical estimates of the 
	Bayes error rate are computed from the minimal spanning tree (MST) of 
	the samples from each pair of classes. Moreover, a meta-learning 
	technique is presented for quantifying the one-vs-rest Bayes error 
	rate for each individual class from a single MST on the entire 
	dataset.  Extensive simulations on benchmark datasets show that the 
	proposed hierarchical method can often be learned much faster than 
	competing methods, while achieving competitive accuracy.
	\vskip\baselineskip\noindent
\end{abstract}

\section{Introduction}
The Bayes error rate (BER) is a central concept in the statistical theory of 
classification. It represents the error rate of the Bayes classifier, which 
assigns a label to an object corresponding to the class with the highest 
posterior probability. By definition, the Bayes error represents the smallest 
possible average error rate that can be achieved by any decision rule 
\citep{wald1947foundations}. Because of these properties, the BER is of great 
interest both for benchmarking classification algorithms as well as for the 
practical design of classification algorithms. For example, an accurate 
approximation of the BER can be used for classifier parameter selection, data 
dimensionality reduction, or variable selection. However, accurate BER 
approximation is difficult, especially in high dimension, and thus much 
attention has focused on tight and tractable BER bounds. This paper proposes a 
model-free approach to designing multiclass classifiers using a bias-corrected 
BER bound estimated directly from the multiclass data.

There exists several useful bounds on the BER that are functions of the 
class-dependent feature distributions. These include information theoretic 
divergence measures such as the Chernoff $\alpha$-divergence 
\citep{chernoff1952measure}, the \bat divergence 
\citep{kailath1967divergence}, or the Jensen-Shannon divergence 
\citep{lin1991divergence}.  Alternatively, arbitrarily tight bounds on 
performance can be constructed using sinusoidal or hyperbolic approximations 
\citep{hashlamoun1994tight, avi1996arbitrarily}. These bounds are functions of 
the unknown class-dependent feature distributions.

Recently, \citet{berisha2016empirically} introduced a divergence measure 
belonging to the family of $f$-divergences which tightly bounds the Bayes 
error rate in the binary classification problem. The bounds on the BER 
obtained with this measure are tighter than bounds derived from the \bat or 
Chernoff bounds. Moreover, this divergence measure can be estimated 
nonparametrically from the data without resorting to density estimates of the 
distribution functions. Inspired by the Friedman-Rafsky multivariate runs test 
\citep{friedman1979multivariate}, estimation is based on computing the 
Euclidean minimal spanning tree (MST) of the data, which can be done in 
approximately $O(\nobs \log \nobs)$ time. In this paper we propose 
improvements to this estimator for problems when there are unequal class 
priors and apply the improved estimator to the adaptive design of a 
hierarchical multiclass classifier. Furthermore, a fast method is proposed for 
bounding the Bayes error rate of individual classes which only requires 
computing a single minimal spanning tree over the entire set of samples. Thus 
our proposed method is faster than competing methods that use density plug-in 
estimation of divergence or observed misclassification rates of algorithms, 
such as SVM or logistic regression, which involve expensive parameter tuning. 

%
%
%
%

Quantifying the complexity of a classification problem has been of significant 
interest \citep{ho2002complexity} and it is clear that a fast and accurate 
estimate of this complexity has many practical applications.  For instance, an 
accurate complexity estimator allows the researcher to assess a priori whether 
a given classification problem is \emph{difficult} to classify or not.  In a 
multiclass problem, a pair of classes which are difficult to disambiguate 
could potentially be merged or could be designated for additional data 
collection.  Moreover, an accurate estimate of the BER could be used for 
variable selection, an application that was explored previously in 
\citet{berisha2016empirically}.  In Section~\ref{sec:meta_learning} further 
applications of the BER estimates to multiclass classification are presented 
and evaluated.

There are many methods available for the design of multiclass classification 
algorithms, including: logistic regression \citep{cox1958regression}; support 
vector machines \citep{cortes1995support}; and neural networks 
\citep{mcculloch1943logical}.  It is often the case that classifier 
performance will be better for some classes than for others, for instance due 
to sample imbalance in the training set. Often classifier designs apply 
weights to the different classes in order to reduce the effect of such 
imbalances on average classifier accuracy 
\citep{lu1998robust,qiao2009adaptive}.  We take a different and more general 
approach that incorporates an empirical determination of the relative 
difficulties of classifying between different classes. Accurate empirical 
estimates of the BER are used for this purpose. A multiclass classifier is 
presented in Section~\ref{sec:smartsvm} that uses MST-based BER estimates to 
create a hierarchy of binary subproblems that increase in difficulty as the 
algorithm progresses.  This way, the classifier initially works on easily 
decidable subproblems before moving on to more difficult multiclass 
classification problems.




The paper is organized as follows. The theory of the nonparametric Bayes error 
estimator of \citet{berisha2016empirically} will be reviewed in 
Section~\ref{sec:mich_theory}.  We will introduce a bias correction for this 
estimator, motivate the use of the estimator for multiclass classification, 
and discuss computational complexity. Section~\ref{sec:meta_learning} will 
introduce applications of the estimator to meta-learning in multiclass 
classification. A novel hierarchical classification method will be introduced 
and evaluated in Section~\ref{sec:smartsvm}.  Section~\ref{sec:mich_discuss} 
provides concluding remarks.

\section{An Improved BER Estimator}
\label{sec:mich_theory}

Here the motivation and theory for the estimator of the Bayes error rate is 
reviewed, as introduced by \citet{berisha2016empirically}.  An improvement on 
this estimator is proposed for the case where class prior probabilities are 
unequal.  Next, the application of the estimator in multiclass classification 
problems is considered. Finally, computational considerations and robustness 
analyses are presented.

\subsection{Estimating the Bayes Error Rate}
Consider the binary classification problem on a dataset $\mathcal{D} = 
\{(\bft{x}_i, y_i)\}_{i=1, \ldots, \nobs}$, where $\bft{x}_i \in \mathcal{X} 
\subseteq \mathbb{R}^{\nvar}$ and $y_i \in \{1, 2\}$. Denote the multivariate 
density functions for the two classes by $f_1(\bft{x})$ and $f_2(\bft{x})$ and 
the prior probabilities by $p_1$ and $p_2 = 1 - p_1$, respectively. The Bayes 
error rate of this binary classification problem can be expressed as 
\citep{fukunaga1990statistical}
\begin{equation}
	P_e(f_1, f_2) = \int \min \left\{ p_1 f_1(\bft{x}), p_2 f_2(\bft{x})  
	\right\} \myd \bft{x}.
\end{equation}
%

Recently, \citet{berisha2016empirically} derived a tight bound on the BER 
which can be estimated directly from the data without a parametric model for 
the density or density estimation.  This bound is based on a divergence 
measure introduced by \citet{berisha2015empirical}, defined as
\begin{equation}
	\label{eq:hpdiv}
	D_{HP}(f_1, f_2) = \frac{1}{4p_1p_2}\left[ \int \frac{ \left( p_1 
				f_1(\bft{x}) - p_2 f_2(\bft{x}) \right)^2 }{ 
			p_1f_1(\bft{x}) + p_2f_2(\bft{x})  } \myd \bft{x} - 
		(p_1 - p_2)^2 \right],
\end{equation}
and called the Henze-Penrose divergence, as it is motivated by an affinity 
measure defined by \citet{henze1999multivariate}. In 
\citet{berisha2015empirical} it was shown that (\ref{eq:hpdiv}) is a proper 
$f$-divergence as defined by \citet{csiszar1975divergence}.

Estimation of the Henze-Penrose (HP) divergence is based on the multivariate 
runs test proposed by \citet{friedman1979multivariate} and convergence of this 
test was studied by \citet{henze1999multivariate}. Let $\bft{X}_k = \{ 
\bft{x}_i \in \mathcal{X} \subseteq \mathbb{R}^{\nvar} : y_i = k \}$ with $k 
\in \{1, 2\}$ denote multidimensional features from two classes. Define the 
class sample sizes $\nobs_k = \abs{\bft{X}_k}$, $k \in \{1, 2\}$.  Let the 
combined sample be denoted by $\bft{X} = \bft{X}_1 \cup \bft{X}_2$ where 
$\nobs = \abs{\bft{X}} = \nobs_1 + \nobs_2$ is the total number of samples 
from both classes.  Define the complete graph $\mathcal{G}$ over $\mathbf{X}$ 
as the graph connecting all $\nobs$ nodes $\{\bft{x}_i\}_{i=1}^{\nobs}$ with 
edge weights $|e_{kj}| = \norm{\bft{x}_k - \bft{x}_j}$ equal to Euclidean 
distances.  The Euclidean minimal spanning tree that spans $\bft{X}$, denoted 
by $\mathcal{T}$, is defined as the subgraph of $\mathcal{G}$ that is both 
connected and whose sum of edge weights is the smallest possible.  The 
Friedman-Rafsky test statistic equals the number of edges in $\mathcal{T}$ 
that connect an observation from class $1$ to an observation from class $2$ 
and is denoted by $R_{1,2}$.

Building on the work by \citet{henze1999multivariate}, 
\citet{berisha2016empirically} show that if $\nobs_1 \rightarrow \infty$ and 
$\nobs_2 \rightarrow \infty$ in a linked manner such that $\nobs_1/(\nobs_1 + 
\nobs_2) \rightarrow \delta \in (0, 1)$ then,
\begin{equation}
	\label{eq:convergence_dhp}
	1 - \frac{\nobs R_{1, 2}}{2 \nobs_1 \nobs_2} \rightarrow D_{HP}(f_1, 
	f_2) \qquad \textit{a.s.}
\end{equation}
Thus, the number of cross connections between the classes in the Euclidean MST 
is inversely proportional to the divergence between the respective probability 
density functions of these classes. 

Finally, the HP-divergence can be used to bound the Bayes error rate, 
$P_e(f_1, f_2)$, following Theorem~$2$ of \citet{berisha2016empirically}
\begin{equation}
	\tfrac{1}{2} - \tfrac{1}{2}\sqrt{u_{HP}(f_1, f_2)} \leq P_e(f_1, f_2) 
	\leq \tfrac{1}{2} - \tfrac{1}{2}u_{HP}(f_1, f_2),
\end{equation}
where
\begin{equation}
	u_{HP}(f_1, f_2) = 4 p_1 p_2 D_{HP}(f_1, f_2) + (p_1 - p_2)^2.
\end{equation}
Averaging these bounds yields an estimate of the BER given by
\begin{equation}
	\label{eq:hp_estimator}
	\hat{P}_e(f_1, f_2) = \tfrac{1}{2} - \tfrac{1}{4}\sqrt{u_{HP}(f_1, 
		f_2)} - \tfrac{1}{4}u_{HP}(f_1, f_2).
\end{equation}
In the following, this estimator will be referred to as the HP-estimator of 
the BER.

\subsection{A modified HP-estimator for unequal class priors}
\label{sec:unequalprior}

To illustrate the performance of the HP-estimator, and to motivate the 
proposed modification, consider a binary classification problem where the 
samples are drawn from two independent bivariate Gaussian distributions with 
equal covariance matrices. For this example the BER and associated bounds can 
be computed exactly \citep{fukunaga1990statistical}. In 
Figure~\ref{fig:error_illustration_1} we compare the BER, the HP-estimator of 
the BER (\ref{eq:hp_estimator}) and the popular \bat bound on the BER 
\citep{bhattacharyya1946measure,kailath1967divergence}.  
Figure~\ref{fig:error_illustration_1} shows that the HP-estimator is closer to 
the true BER than the \bat bound. This result was illustrated by 
\citet{berisha2016empirically} for the case where $p_1 = p_2$ and is confirmed 
here for the case where $p_1 \neq p_2$.

\begin{figure}[tb]
	\def\FigWidthHPAccuracy{0.40\textwidth}
	\centering
	\subfloat[][without bias correction]{%
		\includegraphics[width=\FigWidthHPAccuracy]{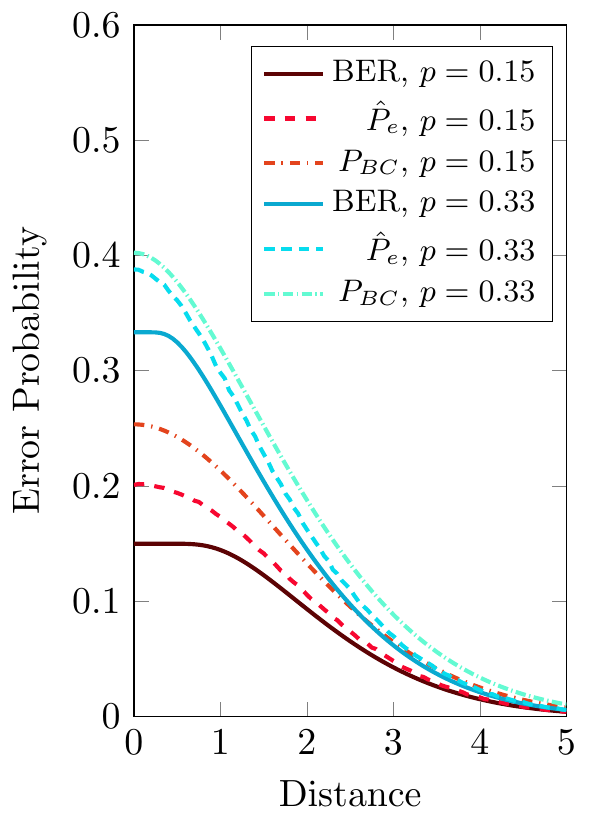}
		\label{fig:error_illustration_1}
	}\qquad\qquad
	\subfloat[][with bias correction]{%
		\includegraphics[width=\FigWidthHPAccuracy]{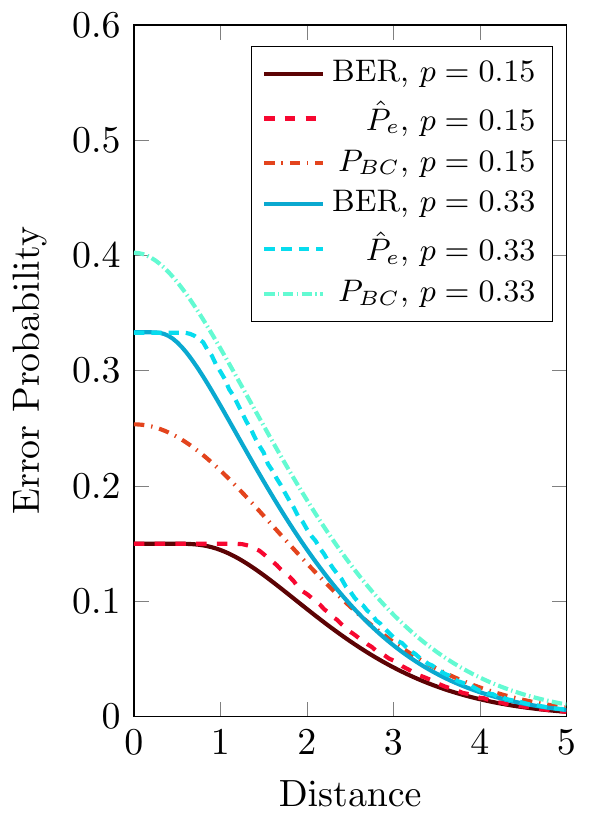}
		\label{fig:error_illustration_2}
	}
	\caption[Bias correction on the BER estimate]{Illustration of the 
		estimates of the Bayes error rate for two different values of 
		the prior probability $p$ and for spherical bivariate Gaussian 
		distributions with increasing separation between the classes 
		as measured by Euclidean distance between the class means, 
		$\norm{\bfs{\mu}_1 - \bfs{\mu}_2}$. We compare the true BER, 
		the \bat bound ($P_{BC}$) requiring the true class 
		distributions, and the HP-estimator ($\hat{P}_e$) based on an 
		empirical sample.  Figures~\subref{fig:error_illustration_1} 
		and \subref{fig:error_illustration_2} show the estimates for 
		$p = 0.15$ and $p = 1/3$, respectively with and without the 
		proposed bias correction of the HP-estimator. The HP-estimator 
		was implemented using $\nobs = 1000$ samples ($\nobs_1 = 
		p\nobs$ and $\nobs_2 = \nobs - \nobs_1$) and averaged over 200 
		trials.  \label{fig:error_illustration}}
\end{figure}

However, Figure~\ref{fig:error_illustration_1} also shows that a significant 
bias occurs in the HP-estimate of the BER when the distance between classes is 
small. Considering that $P_e(f_1, f_2) = \min\{p_1, p_2\}$ if $f_1 \rightarrow 
f_2$, the solution of the equation $\hat{P}_e(f_1, f_2) = \min\{\hat{p}_1, 
\hat{p}_2\}$ for $R_{1,2}$ suggests a bias corrected version of the 
HP-estimate:\footnote{This correction can readily be derived by using the fact 
	that $1 - \frac{2R_{1,2}}{n} \rightarrow u_{HP}(f_1, f_2)$, which 
	holds under the same conditions as for (\ref{eq:convergence_dhp}).}
\begin{equation}
	\label{eq:bias_correction}%
	R_{1,2}' = \min\{\gamma, R_{1,2}\},
\end{equation}
with
\begin{equation}
	\gamma = 2 \nobs \min\{\hat{p}_1, \hat{p}_2\} - \tfrac{3}{4}\nobs + 
	\tfrac{1}{4} \nobs \sqrt{9 - 16 \min\{\hat{p}_1, \hat{p}_2\}},
\end{equation}
where $\hat{p}_1 = \nobs_1 / \nobs$ and $\hat{p}_2 = \nobs_2 / \nobs$ are 
estimates of the true prior probabilities.  
Figure~\ref{fig:error_illustration_2} shows the effect of this bias correction 
on the accuracy of the HP-estimator. As can be seen, the bias correction 
significantly improves the accuracy of the HP-estimator for when the class 
distributions are not well separated.

\subsection{Multiclass classification}
Here we apply the HP-estimate to multiclass classification problems by 
extending the bias corrected HP-estimator to a multiclass Bayes error rate.  
The original multiclass HP-estimator has been defined by 
\citet{wisler2016empirically} and we show how the framework can be applied to 
hierarchical multiclassifier design.

Consider a multiclass problem with $\nclass$ classes with $\bft{x}_i \in 
\mathcal{X} \subseteq \mathbb{R}^{\nvar}$ and $y_i \in \{1, \ldots, 
\nclass\}$, with prior probabilities $p_k$ and density functions 
$f_k(\bft{x})$ for $k = 1, \ldots, \nclass$ such that $\sum_k p_k = 1$. Then, 
the BER can be estimated for each pair of classes using the bias-corrected 
HP-estimator $\hat{P}_e(f_k, f_j)$ using (\ref{eq:bias_correction}). The 
binary classification problem with the largest BER estimate is defined as most 
difficult.

Recall that the largest BER that can be achieved in a binary classification 
problem with unequal class priors is equal to the value of the smallest prior 
probability. This makes it difficult to compare empirical estimates of the BER 
when class sizes are imbalanced. To correct for this, the HP-estimator for 
pairwise classification BERs can be normalized for class sizes using
\begin{equation}
	\label{eq:hp_normalized}
	\hat{P}_e'(f_k, f_l) = \frac{\hat{P}_e(f_k, f_l)}{\min\{\hat{p}_k, 
		\hat{p}_l\}}.
\end{equation}
This normalization places the HP-estimate in the interval $[0, 1]$ and makes 
it possible to more accurately compare the BER estimates of different binary 
problems.

In practice it can also be of interest to understand how difficult it is to 
discriminate each individual class. By reducing the multiclass problem to a 
One-vs-Rest classification problem, it is straightforward to define a 
\emph{confusion rate} for a given class $k$. This represents the fraction of 
instances that are erroneously assigned to class $k$ and the fraction of 
instances which are truly from class $k$ that are assigned to a different 
class. Formally, define the confusion rate for class $k$ as
\begin{equation}
	C_k(y, \hat{y}) = \frac{ \abs{\{ i : \hat{y}_i = k, y_i \neq k \}} + 
		\abs{\{ i : \hat{y}_i \neq k, y_i = k \}} }{\nobs},
\end{equation}
with $\hat{y}_i$ the predicted class for instance $i$. Recall that the Bayes 
error rate is the error rate of the Bayes classifier, which assigns an 
instance $\bft{x}$ to class $k = \argmax_l p_l f_l(\bft{x})$.  Hence, the BER 
for a single class $k$ equals the error of assigning to a class $l \neq k$ 
when the true class is $k$ and the total error of assigning to class $k$ when 
the true class is $c \neq k$, thus
\begin{equation}
	P_{e, k} =\qquad\qquad 
	\smashoperator{\int\limits_{\strut\max\limits_{l \neq k}\{ p_l 
			f_l(\bft{x})\} \geq p_k f_k(\bft{x})}} p_k 
	f_k(\bft{x}) \myd \bft{x}
	\quad + \quad
	\sum_{c \neq k}\qquad\qquad 
	\smashoperator{\int\limits_{\strut\max\limits_{l \neq k}\{ p_l 
			f_l(\bft{x}) \} < p_k f_k(\bft{x})}} p_c f_c(\bft{x}) 
	\myd \bft{x}
\end{equation}

We make two observations about this \emph{One-vs-Rest} Bayes error rate 
(OvR-BER).  First, the OvR-BER for class $k$ is smaller than the sum of the 
binary BERs for the problems involving class $k$ (see 
Appendix~\ref{app:ovr_bayes}).  Second, the OvR-BER can be estimated using the 
Henze-Penrose divergence with $R(\bft{X}_k, \bigcup_{l \neq k} \bft{X}_l)$, 
which yields the estimate $\hat{P}_{e,k}$. A computational advantage of using 
the OvR-BER in multiclass problems is that the MST only has to be computed 
only once on the set $\bft{X}$, since the union of $\bft{X}_k$ and 
$\cup_{l\neq k} \bft{X}_l$ is equal to $\bft{X}$. Therefore, $R(\bft{X}_k, 
\bigcup_{l \neq k} \bft{X}_l)$ can be computed for all $k$ from the single MST 
on $\bft{X}$ by keeping track of the labels of each instance.

\subsection{Computational Considerations}
\label{sec:computconsid}
The construction of the minimal spanning tree lies at the heart of the 
HP-estimator of the BER, so it is important to use a fast algorithm for the 
MST construction. Since the HP-estimator is based on the Euclidean MST the 
dual-tree algorithm by \citet{march2010fast} can be applied.  This algorithm 
is based on the construction of \citet{boruuvka1926jistem} and implements the 
Euclidean MST in approximately $O(n \log n)$ time.  For larger datasets it can 
be beneficial to partition the space into hypercubes and construct the MST in 
each partition.

A simple way to improve the robustness of the HP-estimator is to use multiple 
orthogonal MSTs and average the number of cross-connections 
\citep{friedman1979multivariate}. Computing orthogonal MSTs is not 
straightforward in the dual-tree algorithm of \citet{march2010fast}, but is 
easy to implement in MST algorithms that use a pairwise distance matrix such 
as that of \citet{whitney1972algorithm}. Figure~\ref{fig:ortho_msts} shows the 
empirical variance of the HP-estimator for different numbers of orthogonal 
MSTs as a function of the separation between the classes. As expected, the 
variance decreases as the number of orthogonal MSTs increases, although the 
benefit of including more orthogonal MSTs also decreases when adding more 
MSTs.  Therefore, $3$ orthogonal MSTs are typically used in practice.

\begin{figure}[tb]
	\centering
	\includegraphics[width=.75\textwidth]{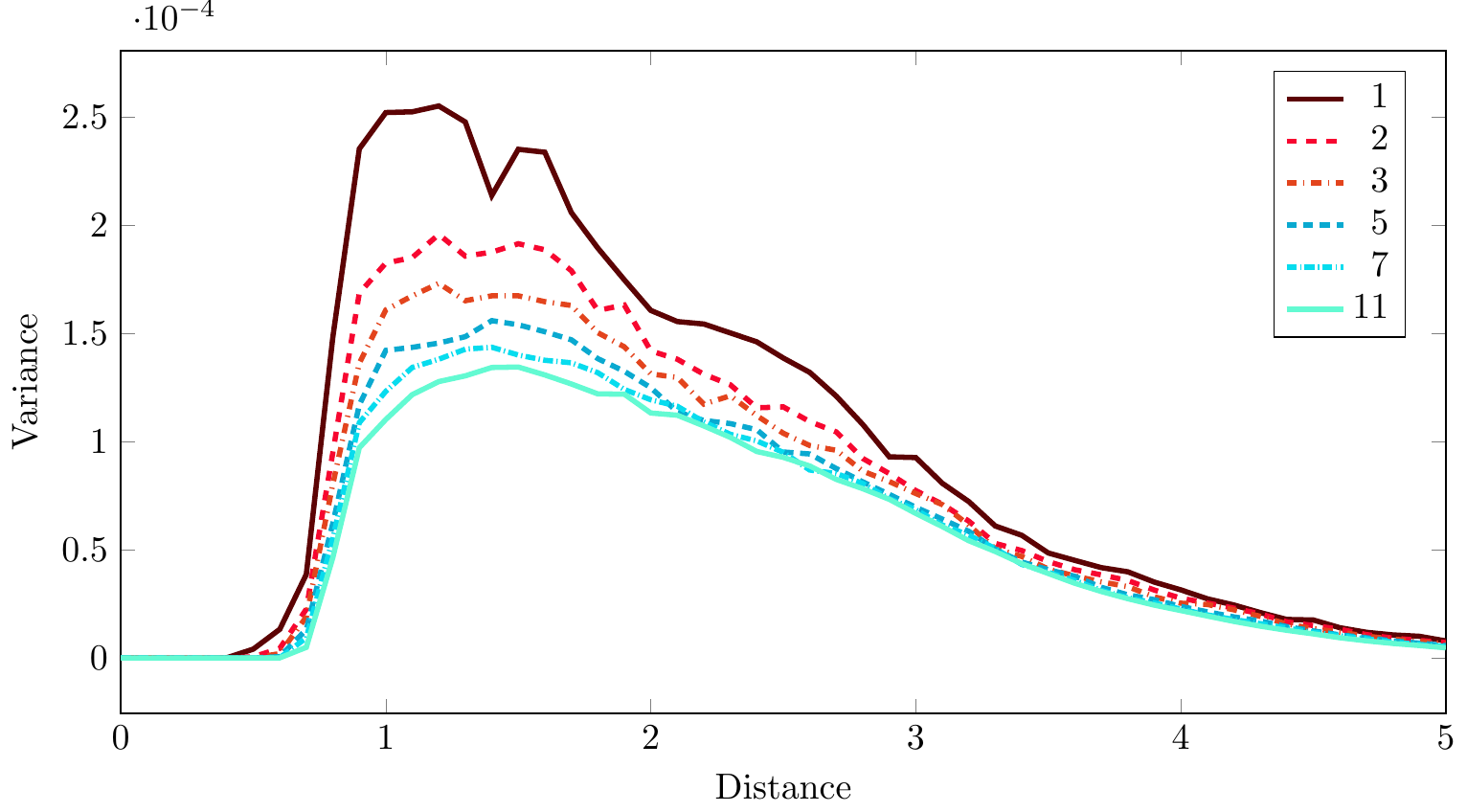}
	\caption[Effect of orthogonal MSTs on BER estimates]{Variance of the 
		HP-estimator of the BER for varying number of orthogonal MSTs 
		as a function of the separation between the classes.  Results 
		are based on 500 repetitions of a binary classification 
		problem with $\nobs = 1000$, $\nvar = 2$, $p_1 = 0.33$, where 
		class distributions are bivariate spherical Gaussians with 
		different means. \label{fig:ortho_msts}}
\end{figure}


\section{Meta-Learning of optimal classifier accuracy}
\label{sec:meta_learning}
Applying the HP-estimator to meta-learning problems creates a number of 
opportunities to assess the difficulty of a classification problem before 
training a classifier. For example, given a multiclass classification problem 
it may be useful to know which classes are difficult to distinguish from each 
other and which classes are easy to distinguish.  Figure~\ref{fig:heatmap} 
shows an illustration of this for handwritten digits in the well-known MNIST 
dataset \citep{lecun1998gradient}. This figure shows a heat map where each 
square corresponds to an estimate of the BER for a binary problem in the 
training set. From this figure it can be seen that the digits 4 and 9 are 
difficult to distinguish, as well as the digits 3 and 5.  This information can 
be useful for the design of a classifier, to ensure for instance that higher 
weights are placed on misclassifications of the more difficult number pairs if 
correct classification of these pairs is of importance to the end-task. In 
Figure~\ref{fig:predmap} a similar heat map is shown based on misclassified 
instances of LeNet-5 \citep{lecun1998gradient} on the test set.  This figure 
shows the symmetric confusion matrix based on the 82 misclassified instances.  
As can be seen, this figure closely corresponds to the heat map on the 
training data, which confirms the predictive accuracy of the HP-estimator for 
real data.

\begin{figure}[tb]
	\def\MNISTHeatWidth{0.35\textwidth}
	\centering
	\subfloat[][]{%
		\includegraphics[width=\MNISTHeatWidth]{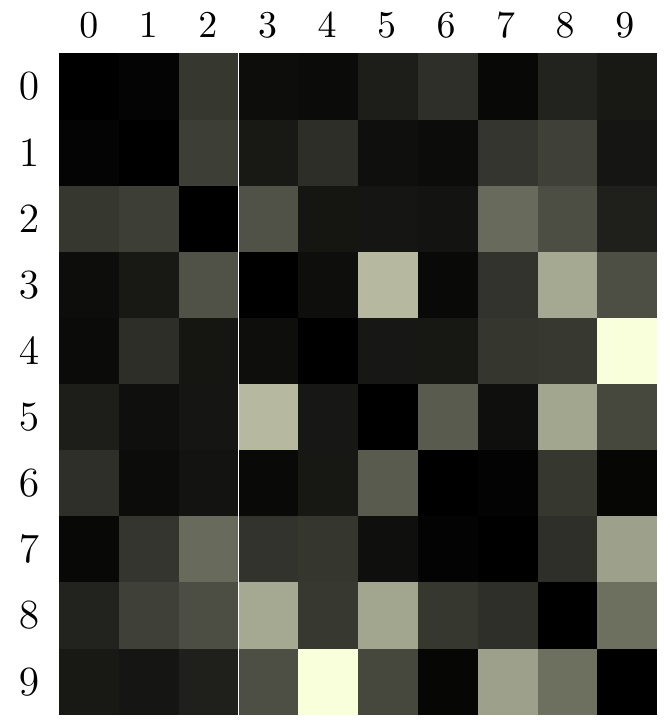}
		\label{fig:heatmap}
	}\qquad\qquad
	\subfloat[][]{%
		\includegraphics[width=\MNISTHeatWidth]{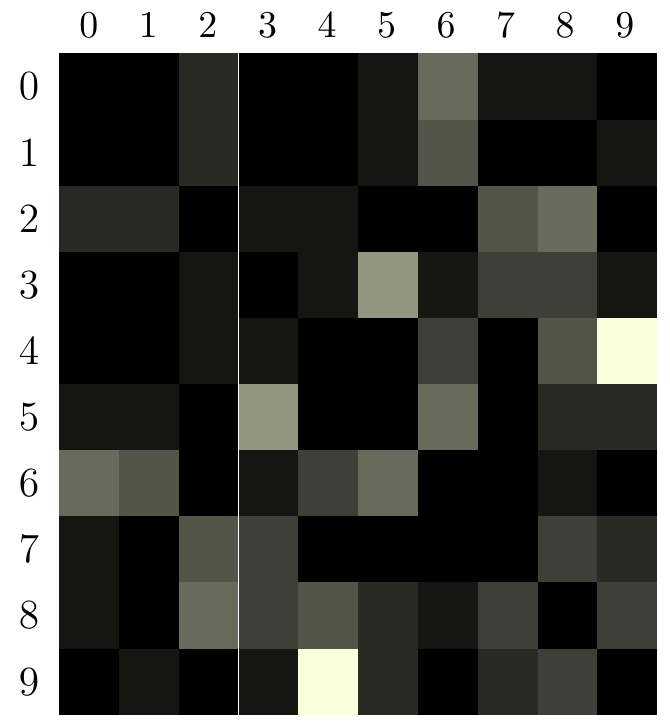}
		\label{fig:predmap}
	}
	\caption[Heat maps of BER estimates on the MNIST dataset]{Heat maps 
		illustrating the difficulty of distinguishing between 
		different handwritten digits in the MNIST dataset of 
		\citet{lecun1998gradient}.  Brighter squares correspond to 
		higher values. In \subref{fig:heatmap} the heat map of the BER 
		estimates on the training data is shown. In 
		\subref{fig:predmap} the heat map is shown of the 82 
		misclassifications made by LeNet-5 on the test data 
		\citep{lecun1998gradient}. From \subref{fig:heatmap} it can be 
		readily identified that the numbers 3 and 5 are difficult to 
		distinguish, as well as 4 and 9. Easy to distinguish number 
		pairs are for instance 6 and 7, and 0 and 1. This pattern is 
		reflected in the results on the test data shown in 
		\subref{fig:predmap}. \label{fig:mnist_heat}}
\end{figure}

Another example of the accuracy of the BER estimates for multiclass 
classification problems is given in Figure~\ref{fig:krkopt}. In this figure, 
OvR-BER estimates $\hat{P}_{e,k}$ and class accuracy scores $C_k(y, \hat{y})$ 
are shown for the Chess dataset ($\nobs = 28056$, $\nvar = 34$, $\nclass = 
18$) obtained from the UCI repository \citep{Bache+Lichman:2013}. This dataset 
was split into a training dataset (70\%) and a test dataset (30\%) and the 
OvR-BER estimates were computed on the training dataset. These estimates are 
compared with the class error rates obtained from out-of-sample predictions of 
the test dataset using GenSVM \citep{van2016gensvm}. This figure shows that 
the OvR-BER estimates are accurate predictors of classification performance.  
The classes that are relatively difficult to classify may benefit from 
increasing misclassification weights.

\begin{figure}[tb]
	\centering
	\includegraphics[width=0.75\textwidth]{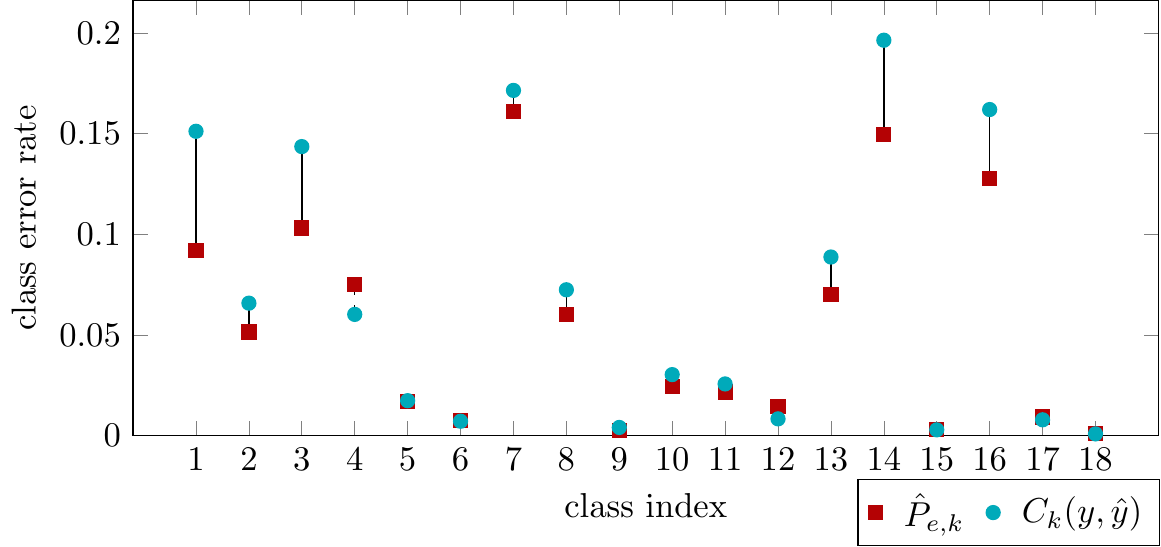}
	\caption[OvR-BER estimates and error rates on the Chess dataset]{
		OvR-BER estimates and test set error rates for each class in 
		the Chess dataset.  It can be seen that for most classes the 
		OvR-BER estimate is an accurate predictor for test set 
		performance.  For classes where there is a difference between 
		the BER and the test set error the BER generally gives a lower 
		bound on the test set performance.  Test set results were 
		obtained with the GenSVM classifier \citep{van2016gensvm}.  
		\label{fig:krkopt}}
\end{figure}

The BER estimates can also be applied to feature selection and, in particular, 
to the identification of useful feature transformations of the data.  A 
feature selection strategy based on forward selection was outlined in 
\citet{berisha2016empirically}. At each feature selection stage, this 
algorithm adds the feature which gives the smallest increase in the BER 
estimate. \citet{berisha2016empirically} show that this feature selection 
strategy quickly yields a subset of useful features for the classification 
problem.


Because the BER estimate is a fast and asymptotically consistent estimate of a 
bound on classification performance, it is easy to try a number of potential 
feature transformations and use the one with the smallest BER estimate in the 
classifier. This can be useful both for traditional feature transformations 
such as PCA \citep{pearson1901liii} and Laplacian Eigenmaps 
\citep{belkin2003laplacian}, but also for commonly used kernel transformations 
in SVMs. For a researcher this can significantly reduce the time needed to 
train a classifier on different transformations of the data. In a multiclass 
setting where the One-vs-One strategy is used, one can even consider a 
different feature transformation for each binary subproblem.  When using a 
unified classification method one can consider feature transformations which 
reduce the average BER estimate or the worst-case BER estimate.

Note that a feature transformation which reduces the dimensionality of the 
dataset without increasing the BER estimate can be considered beneficial, as 
many classification methods are faster for low-dimensional datasets. For 
instance, applying PCA with $2$ components on the Chess dataset only slightly 
increases the BER estimates for two classes, while remaining the same for the 
other classes. Thus, a classifier will likely achieve comparable accuracy with 
this transformed dataset, but will be much faster to train since the 
dimensionality can be reduced from $\nvar = 34$ to $\nvar = 2$.

\section{Hierarchical Multiclass Classification}
\label{sec:smartsvm}
In this section a novel hierarchical multiclass SVM classifier is introduced 
which is based on uncertainty clustering. The BER estimate can be considered a 
measure of the irreducible uncertainty of a classification problem, as a high 
BER indicates an intrinsically difficult problem. This can be used to 
construct a tree of binary classification problems that increase in difficulty 
along the depth of the tree. By fitting a binary classifier (such as an SVM) 
at each internal node of the tree, a classification method is obtained which 
proceeds from the easier binary subproblems to the more difficult binary 
problems.

Similar divide-and-conquer algorithms have been proposed \citep[among 
others]{schwenker2001tree, takahashi2002decision, frank2004ensembles, 
	vural2004hierarchical, tibshirani2007margin}. See 
\citet{lorena2008review} for a review.  These approaches often apply a 
clustering method to create a grouping of the dataset into two clusters, 
repeating this process recursively to form a binary tree of classification 
problems. In \citet{lorena2010building} several empirical distance measures 
are used as indicators of separation difficulty between classes, which are 
applied in a bottom-up procedure to construct a classification tree. Finally, 
in \citet{el2010hierarchical} the Jensen-Shannon divergence is used to bound 
the BER with inequalities from \citet{lin1991divergence} and a classification 
tree is constructed using a randomized heuristic procedure.  Unfortunately, 
the Jensen-Shannon divergence implementation requires parametric estimation of 
distribution functions.  Moreover, for the equiprobable case the upper bound 
on the BER obtained with the Jensen-Shannon divergence can be shown to be less 
tight than that obtained with the HP-divergence (see 
Appendix~\ref{app:jensen_shannon}).  Because of this, these estimates of the 
BER may be less accurate than those obtained with the proposed HP-estimator.

To construct the hierarchical classification tree a complete weighted graph 
$\mathcal{G} = (V, E, w)$ is created where the vertices correspond to the 
classes and the weight of the edges equals the HP-estimate for that binary 
problem.  Formally, let $V = \{1, \ldots, \nclass\}$, $E = \{ (i, j) : i, j 
\in V, i \neq j \}$ and define the edge weight $w(e)$ for $e \in E$ as $w(e) = 
w(i, j) = \hat{P}'_e(f_i, f_j)$.  In the HP-estimator $\hat{P}'_e(f_i, f_j)$ 
the bias correction (\ref{eq:bias_correction}) and the normalization 
(\ref{eq:hp_normalized}) are used. By recursively applying min-cuts to this 
graph a tree of binary classification problems is obtained which increase in 
difficulty along the depth of the tree. Min-cuts on this weighted graph can be 
computed using for instance the method of \citet{stoer1997simple}.  
Figure~\ref{fig:graph_cut} illustrates this process for a multiclass 
classification problem with $\nclass = 7$.

\begin{figure}[tb]
	\centering
	\def\GraphCutFigWidth{.40\textwidth}
	\subfloat[][]{%
		\includegraphics[width=\GraphCutFigWidth]{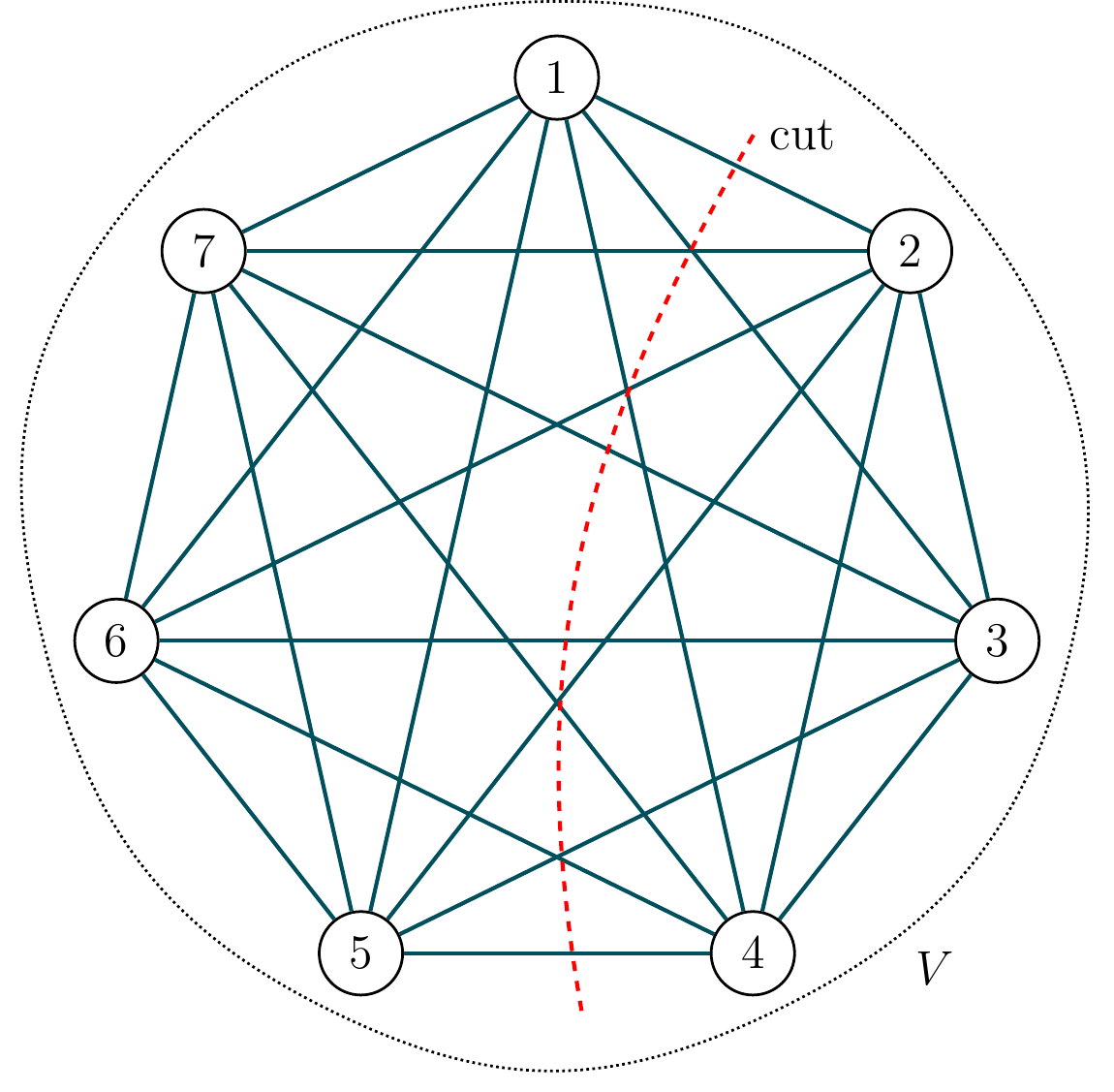}
		\label{fig:graph_cut_1}
	}\qquad
	\subfloat[][]{%
		\includegraphics[width=\GraphCutFigWidth]{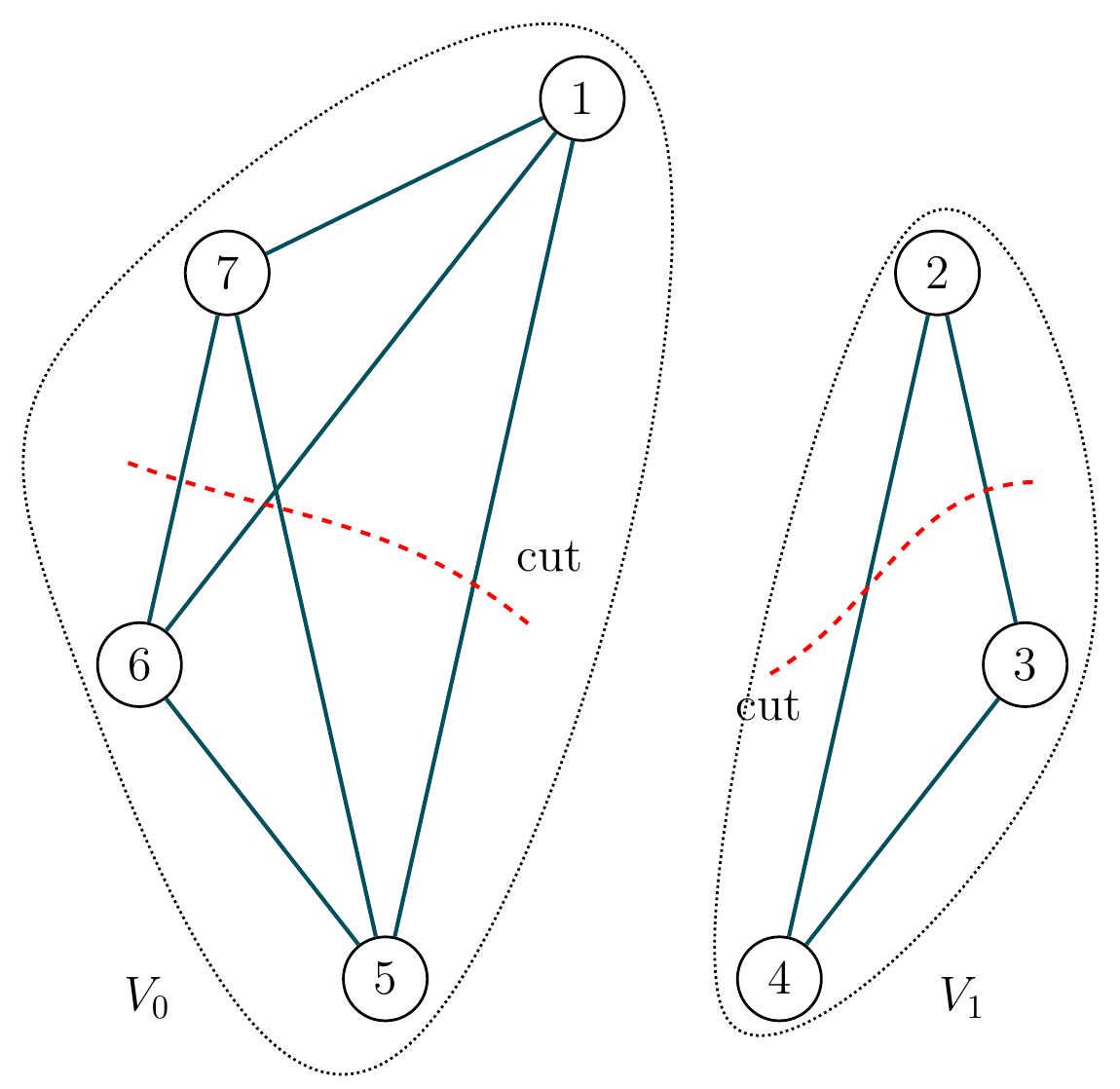}
		\label{fig:graph_cut_2}
	}
	\caption[Illustration of the min-cut procedure]{Illustration of the 
		min-cut procedure for a multiclass classification problem with 
		$\nclass = 7$ classes.  The first cut in 
		\subref{fig:graph_cut_1} splits the vertices in the set $V$ 
		into two sets, $V_{0}$ and $V_{1}$, which are used in a binary 
		classifier.  Figure~\subref{fig:graph_cut_2} illustrates the 
		next step where both $V_{0}$ and $V_{1}$ are split further by 
		cuts in the complete graph for each subset.  
		\label{fig:graph_cut}}
\end{figure}

The tree construction can be described formally as follows. Starting with the 
complete weighted graph with vertices $V$, apply a min-cut algorithm to obtain 
the disjoint vertex sets $V_{0}$ and $V_{1}$ such that $V_0 \cup V_1 = V$.  
This pair of vertex sets then forms a binary classification problem with 
datasets $\bft{X}_{0} = \{ \bft{x}_i \in \bft{X} : y_i \in V_{0} \}$ and 
$\bft{X}_{1} = \{ \bft{x}_i \in \bft{X} : y_i \in V_{1} \}$. Recursively 
applying this procedure to the sets $V_{0}$ and $V_{1}$ until no further 
splits are possible yields a tree of binary classification problems, as 
illustrated in Figure~\ref{fig:class_tree}.

\begin{figure}[tb]
	\centering
	\includegraphics[width=0.75\textwidth]{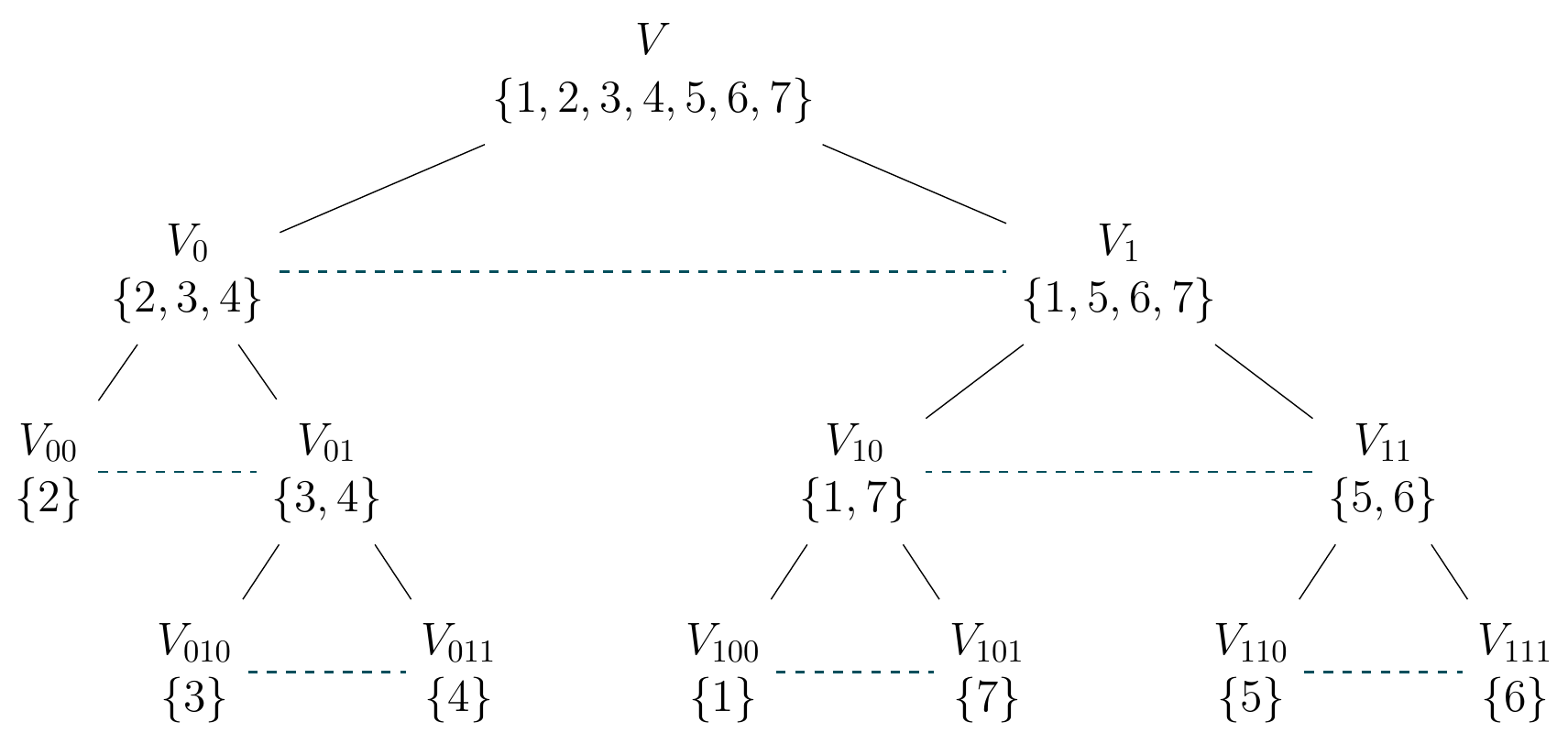}
	\caption[Illustration of tree induced by graph cutting]{
		Illustration of the tree induced by the graph cutting 
		procedure illustrated in Figure~\ref{fig:graph_cut}, for a 
		hypothetical problem with $\nclass = 7$ classes. Each dashed 
		line in the tree indicates a point where a binary classifier 
		will be trained. \label{fig:class_tree}}
\end{figure}

In the remainder of this section the results of an extensive simulation study 
are presented, which aims to evaluate the performance of this hierarchical 
classifier on multiclass classification problems. The classifier that will be 
used in each binary problem in the tree will be a linear support vector 
machine, but in practice any binary classifier could be used in the algorithm.  
The implementation of the hierarchical classifier based on the linear binary 
SVM will be called \emph{SmartSVM}.\footnote{The SmartSVM classifier and the 
	meta-learning and BER estimation techniques presented in the previous 
	sections have been implemented in the \texttt{smartsvm} Python 
	package, available at: 
	\url{https://github.com/HeroResearchGroup/SmartSVM}.}

The experimental setup is comparable to that used in \citet{van2016gensvm}, 
where a nested cross-validation (CV) approach is used to reduce bias in 
classifier performance \citep{stone1974cross}. From each original dataset 5 
independent training and test datasets are generated.  Subsequently, each 
classification method is trained using 10 fold CV on each of the training 
datasets. Finally, the model is retrained on the entire training dataset using 
the optimal hyperparameters and this model is used to predict the test set. In 
the experiments 16 datasets are used of varying dimensions from which 80 
independent test sets are constructed.  The train and test datasets were 
generated using a stratified split, such that the proportions of the classes 
correspond to those in the full dataset.  Table~\ref{tab:mdatasets} shows the 
descriptive statistics of each of the datasets used. Datasets are collected 
from the UCI repository \citep{Bache+Lichman:2013} and the KEEL repository 
\citep{alcala2010keel}.

\begin{table}[tb]
	\centering
	\caption[Dataset summary statistics]{Dataset summary statistics for 
		the datasets used in the experimental study.  The final two 
		columns denote the size of the smallest and the largest class, 
		respectively.  Datasets marked with an asterisk are collected 
		from the KEEL dataset repository, all others are from the UCI 
		repository. \label{tab:mdatasets}}
	\vskip\baselineskip
	\begin{tabular}{lrrrrr}
		Dataset & Instances & Features & Classes & $\min \nobs_k$ & 
		$\max \nobs_k$ \\
		\hline
		abalone    &   4177 &   8 & 20 &   14 &   689\\
		fars*      & 100959 & 338 &  7 &  299 & 42116\\
		flare      &   1066 &  19 &  6 &   43 &   331\\
		kr-vs-k    &  28056 &  34 & 18 &   27 &  4553\\
		letter     &  20000 &  16 & 26 &  734 &   813\\
		nursery    &  12960 &  19 &  4 &  330 &  4320\\
		optdigits  &   5620 &  64 & 10 &  554 &   572\\
		pageblocks &   5473 &  10 &  5 &   28 &  4913\\
		pendigits  &  10992 &  16 & 10 & 1055 &  1144\\
		satimage   &   6435 &  36 &  6 &  626 &  1533\\
		segment    &   2310 &  19 &  7 &  330 &   330\\
		shuttle    &  58000 &   9 &  6 &   23 & 45586\\
		texture*   &   5500 &  40 & 11 &  500 &   500\\
		wine red   &   1599 &  12 &  5 &   18 &   681\\
		wine white &   4898 &  12 &  6 &   20 &  2198\\
		yeast      &   1479 &   8 &  9 &   20 &   463\\
		\hline
	\end{tabular}
\end{table}

SmartSVM will be compared to five other linear multiclass SVMs in these 
experiments. Three of these alternatives are heuristic methods which use the 
binary SVM as underlying classifier, while two others are single-machine 
multiclass SVMs.  One of the most commonly used heuristic approaches to 
multiclass SVMs is the One vs.  One (OvO) method \citep{kressel1999pairwise} 
which solves a binary SVM for each of the $\nclass (\nclass - 1)$ pairs of 
classes. An alternative is the One vs.  Rest (OvR) method 
\citep{vapnik1998statistical} in which a binary SVM is solved for each of the 
$\nclass - 1$ binary problems obtained by separating a single class from the 
others. The directed acyclic graph (DAG) SVM was proposed by 
\citet{platt2000large} as an extension of the OvO approach. It has a similar 
training procedure as OvO, but uses a different prediction strategy.  In the 
OvO method a voting scheme is used where the class with the most votes from 
each binary classifier becomes the predicted label. In contrast, the DAGSVM 
method uses a voting scheme where the least likely class is voted away until 
only one remains. Finally, two single-machine multiclass SVMs are also 
compared: the method by \citet{crammer2002algorithmic} and GenSVM 
\citep{van2016gensvm}. 

All methods are implemented in either \verb+C+ or \verb|C++|, to ensure that 
speed of the methods can be accurately compared. The methods that use a binary 
SVM internally are implemented with LibLinear \citep{fan2008liblinear}.  
LibLinear also implements a fast solver for the method by 
\citet{crammer2002algorithmic} using the algorithm proposed by 
\citet{keerthi2008sequential}.  For SmartSVM the Bayes error rates and the 
corresponding classification tree were calculated once for each training 
dataset as a preprocessing step.  For most datasets the BERs were computed 
based on 3 orthogonal MSTs using the algorithm of 
\citet{whitney1972algorithm}. For the two largest datasets (fars and shuttle) 
the BER was computed based on a single MST using the algorithm of 
\citet{march2010fast}.  Computing these MSTs was done in parallel using at 
most 10 cores. In the results on training time presented below the training 
time of SmartSVM is augmented with the preprocessing time.

The binary SVM has a cost parameter for the regularization term, which is 
optimized using cross validation. The range considered for this parameter is 
$C \in \{2^{-18}, 2^{-16}, \ldots, 2^{18}\}$. The GenSVM method has additional 
hyperparameters which were varied in the same way as in the experiments of 
\citet{van2016gensvm}. All experiments were performed on the Dutch National 
LISA Compute Cluster using the \verb+abed+ utility.\footnote{See 
	\url{https://github.com/GjjvdBurg/abed}.}

\begin{table}[bt]
	\centering
	\caption[Training time SmartSVM experiments]{Total training time per 
		dataset in seconds, averaged over the 5 nested CV folds.  
		Minimal values per dataset are underlined.  As can be seen, 
		SmartSVM is the fastest method on 10 out of 16 
		datasets.\label{tab:total_time}}
	\vskip\baselineskip
	\begin{tabular}{lrrrrrr}
		Dataset & C \& S & DAG & GenSVM & OvO & OvR & SmartSVM \\
		\hline
		abalone & 13332 & 146.5 & 86787 & 134.3 & 214.4 & 
		\underline{75.0} \\
		fars & 158242 & 3108 & 205540 & \underline{2866} & 5630 & 3158 
		\\
		flare & 467.5 & 7.0 & 501.7 & 6.5 & 8.2 & \underline{5.0} \\
		krkopt & 86000 & 728.4 & 56554 & 680.4 & 1388 & 
		\underline{664.0} \\
		letter & 31684 & 407.3 & 44960 & \underline{381.5} & 1805 & 
		652.1 \\
		nursery & 2267 & 74.7 & 1363 & 68.7 & 94.9 & \underline{56.8} 
		\\
		optdigits & 88.0 & 20.0 & 31615 & 18.5 & 24.3 & 
		\underline{10.9} \\
		pageblocks & 523.4 & 26.8 & 2001 & \underline{25.2} & 48.6 & 
		27.6 \\
		pendigits & 2499 & 30.2 & 6591 & \underline{28.5} & 151.3 & 
		53.7 \\
		satimage & 5184 & 74.7 & 2563 & \underline{70.0} & 188.4 & 
		74.5 \\
		segment & 377.1 & 7.1 & 1542 & 7.0 & 18.7 & \underline{5.4} \\
		shuttle & 7658 & 570.8 & 14618 & \underline{561.9} & 1996 & 
		832.0 \\
		texture & 588.1 & 34.5 & 14774 & 32.8 & 69.6 & 
		\underline{19.3} \\
		winered & 706.4 & 12.3 & 542.7 & 11.1 & 18.3 & \underline{8.8} 
		\\
		winewhite & 2570 & 59.2 & 2570 & 55.1 & 70.9 & 
		\underline{36.7} \\
		yeast & 1209 & 16.7 & 1138 & 15.3 & 24.2 & \underline{13.2} \\
		\hline
	\end{tabular}
\end{table}

The experiments are compared on training time and out-of-sample predictive 
performance.  Table~\ref{tab:total_time} shows the results for training time, 
averaged over the 5 nested cross validation folds for each dataset. As can be 
seen SmartSVM is the fastest method on 10 out of 16 datasets. This can be 
attributed to the smaller number of binary problems that SmartSVM needs to 
solve compared to OvO and the fact that the binary problems are smaller than 
those solved by OvR. The OvO method is the fastest classification method on 
the remaining 6 datasets. The single-machine multiclass SVMs by 
\citet{crammer2002algorithmic} and \citet{van2016gensvm} both have larger 
computation times than the heuristic methods. Since GenSVM has a larger number 
of hyperparameters, it is interesting to look at the average time per 
hyperparameter configuration as well. In this case, GenSVM is on average 
faster than \citet{crammer2002algorithmic} due to the use of warm starts (see 
Appendix~\ref{app:mich_additional} for additional simulation results).

\begin{table}[bt]
	\centering
	\caption[Predictive Performance SmartSVM experiments]{Out-of-sample 
		predictive performance as measured by the adjusted Rand index.  
		Although SmartSVM doesn't often achieve the maximum 
		performance, there are several datasets for which SmartSVM 
		outperforms One vs. One, or differs from the maximum 
		performance in only the second or third decimal.  
		\label{tab:pred_perf_ari}}
	\vskip\baselineskip
	\begin{tabular}{lrrrrrr}
		Dataset & C \& S & DAG & GenSVM & OvO & OvR & SmartSVM \\
		\hline
		abalone & 0.0788 & 0.0898 & 0.0895 & \underline{0.0898} & 
		0.0791 & 0.0595 \\
		fars & 0.8127 & 0.8143 & 0.8085 & \underline{0.8146} & 0.8134 
		& 0.8115 \\
		flare & 0.4274 & 0.6480 & \underline{0.6687} & 0.6496 & 0.4084 
		& 0.6544 \\
		krkopt & 0.1300 & 0.1988 & 0.1779 & \underline{0.2022} & 
		0.1512 & 0.1585 \\
		letter & 0.6173 & 0.6991 & 0.5909 & \underline{0.7118} & 
		0.5155 & 0.4533 \\
		nursery & 0.8279 & 0.8175 & \underline{0.8303} & 0.8157 & 
		0.8095 & 0.8138 \\
		optdigits & 0.9721 & \underline{0.9854} & 0.9732 & 0.9850 & 
		0.9686 & 0.9640 \\
		pageblocks & \underline{0.7681} & 0.7335 & 0.6847 & 0.7353 & 
		0.6696 & 0.7214 \\
		pendigits & 0.9068 & 0.9566 & 0.8971 & \underline{0.9597} & 
		0.8607 & 0.8817 \\
		satimage & 0.7420 & 0.7652 & 0.7403 & \underline{0.7672} & 
		0.7107 & 0.7607 \\
		segment & \underline{0.9013} & 0.9000 & 0.8812 & 0.8982 & 
		0.8354 & 0.8986 \\
		shuttle & \underline{0.9275} & 0.8543 & 0.8887 & 0.8543 & 
		0.6925 & 0.8038 \\
		texture & 0.9936 & \underline{0.9950} & 0.9888 & 0.9947 & 
		0.9865 & 0.8977 \\
		winered & \underline{1.0000} & \underline{1.0000} & 0.9985 & 
		\underline{1.0000} & 0.8369 & \underline{1.0000} \\
		winewhite & \underline{1.0000} & \underline{1.0000} & 0.9998 & 
		\underline{1.0000} & 0.8131 & \underline{1.0000} \\
		yeast & \underline{0.2595} & 0.2540 & 0.2521 & 0.2587 & 0.2433 
		& 0.2088 \\
		\hline
	\end{tabular}
\end{table}

Classification performance of the methods is reported using the adjusted Rand 
index (ARI) which corrects for chance \citep{hubert1985comparing}. Use of this 
index as a classification metric has been proposed previously by 
\citet{santos2009use}.  Table~\ref{tab:pred_perf_ari} shows the predictive 
performance as measured with the ARI. As can be seen, SmartSVM obtains the 
maximum performance on two of the sixteen datasets. However, SmartSVM 
outperforms One vs. One on 3 datasets and outperforms One vs. Rest on 10 out 
of 16 datasets. The OvO and OvR methods are often used as default heuristic 
approaches for multiclass SVMs and are respectively the default strategies in 
the popular LibSVM \citep{CC01a} and LibLinear \citep{fan2008liblinear} 
libraries.  Since SmartSVM is often faster than these methods, our results 
indicate a clear practical benefit to using SmartSVM for multiclass 
classification.

\section{Discussion}
\label{sec:mich_discuss}

In this work the practical applicability of nonparametric Bayes error 
estimates to meta-learning and hierarchical classifier design has been 
investigated. For the BER estimate introduced by 
\citet{berisha2016empirically} a bias correction was derived which improves 
the accuracy of the estimator for classification problems with unequal class 
priors.  Furthermore, a normalization term was proposed which makes the BER 
estimates comparable in multiclass problems. An expression of the OvR-BER was 
given which represents the exact Bayes error for a single class in the 
multiclass problem and it was shown that this error can be efficiently 
estimated using the HP-estimator as well. A robustness analysis of the 
HP-estimator was performed which showed the benefit of using orthogonal MSTs 
in the estimator.

There are many potential applications of the BER estimates to meta-learning 
problems. Above, several possibilities were explored including the prediction 
of which pairs of classes are most difficult to distinguish and which 
individual classes will yield the highest error rate. Preliminary experiments 
with feature transformations were also performed, which showed that the BER 
estimates can be a useful tool in determining beneficial transformations 
before a classifier is trained.

Based on the weighted graph of pairwise BER estimates, a hierarchical 
multiclass classification method was proposed. The classifier uses a top-down 
splitting approach to create a tree of binary classification problems which 
increase in difficulty along the depth of the tree. By using a linear SVM for 
each classification problem, a hierarchical multiclass SVM was obtained which 
was named SmartSVM. Extensive simulation studies showed that SmartSVM is often 
faster than existing approaches and yields competitive predictive performance 
on several datasets.

Note that the SmartSVM classifier is only one example of how the BER estimates 
can be used to construct better classification methods. As discussed in 
Section~\ref{sec:meta_learning}, BER estimates could also be used to define 
class weights in a multiclass classifier.  Moreover, the min-cut strategy used 
for SmartSVM may not be the optimal way to construct the classification tree.  
Evaluating different approaches to constructing classification hierarchies and 
other applications of the BER estimates to multiclass classification problems 
are topics for further research.

\section*{Acknowledgements}
The computational experiments of this work were performed on the Dutch 
National LISA Compute Cluster, and supported by the Dutch National Science 
Foundation (NWO).  The authors thank SURFsara (\url{www.surfsara.nl}) for the 
support in using the LISA cluster. This research was partially supported by 
the US Army Research Office, grant W911NF-15-1-0479 and US Dept. of Energy 
grant DE-NA0002534.

\appendix

\section{One vs. Rest Bayes Error Rate}
\label{app:ovr_bayes}

In this section bounds for the One vs. Rest Bayes error rate will be derived, 
which measures the error of the Bayes classifier in correctly identifying an 
individual class.

\begin{definition}[OvR-BER]
	Let $f_1, f_2, \ldots, f_{\nclass}$ and $p_1, p_2, \ldots, 
	p_{\nclass}$ denote density functions and prior probabilities for the 
	classes 1 through $\nclass$ respectively, with $\sum_c p_c = 1$.  
	Then, the Bayes error rate between a class $k$ and the remaining 
	classes is given by
	\begin{equation}
		P_{e, k} =\qquad\qquad 
		\smashoperator{\int\limits_{\strut\max\limits_{l \neq k}\{ p_l 
				f_l(\bft{x})\} \geq p_k f_k(\bft{x})}} p_k 
		f_k(\bft{x}) \myd \bft{x}
		\quad + \quad
		\sum_{c \neq k}\qquad\qquad
		\smashoperator{\int\limits_{\strut\max\limits_{l \neq k}\{ p_l 
				f_l(\bft{x}) \} < p_k f_k(\bft{x})}} p_c 
		f_c(\bft{x}) \myd \bft{x}.
	\end{equation}
\end{definition}

Below it will be shown that the OvR-BER can be bounded using the 
Friedman-Rafsky statistic in the One-vs-Rest setting, $R(\bft{X}_k, \bigcup_{l 
	\neq k} \bft{X}_l)$. Let the mixture distribution of the classes $l 
\neq k$ be given by
\begin{equation}
	g_k(\bft{x}) = \frac{\sum_{l \neq k} p_l f_l(\bft{x})}{\sum_{l \neq k} 
		p_l},
\end{equation}
with prior probability $p_g = \sum_{l \neq k} p_l$. Then $\bft{X}_{g_k} = 
\bigcup_{l \neq k} \bft{X}_l$ can be seen as a draw from this mixture 
distribution. By Theorem~$2$ of \citet{berisha2016empirically} it holds that
\begin{equation}
	\tfrac{1}{2} - \tfrac{1}{2}\sqrt{u(f_k, g_k)} \leq \int \min\{p_k 
	f_k(\bft{x}), p_g g_k(\bft{x}) \} \myd \bft{x} \leq \tfrac{1}{2} - 
	\tfrac{1}{2} u(f_k, g_k).
\end{equation}
The following theorem relates this error to the OvR-BER defined above.

\begin{theorem}
	The error rate between class $k$ and the mixture distribution without 
	class $k$ is bounded from above by the OvR-BER,
	\begin{equation}
		Q_{e, k} = \int \min\{p_k f_k(\bft{x}), p_g g_k(\bft{x}) \} 
		\myd \bft{x} \leq P_{e, k}.
	\end{equation}
\end{theorem}
\begin{proof}
	Note that
	\begin{equation}
		Q_{e, k} = \quad \smashoperator{\int\limits_{p_k f_k(\bft{x}) 
				\leq p_g g_k(\bft{x})}} p_k f_k(\bft{x}) \myd 
		\bft{x} \quad + \quad \smashoperator{\int\limits_{p_k 
				f_k(\bft{x}) > p_g g_k(\bft{x})}} p_g 
		g_k(\bft{x}) \myd \bft{x}.
	\end{equation}
	To simplify the notation, introduce the sets
	\begin{align}
		T &= \left\{\bft{x} \in \mathbb{R}^{\nvar} : p_k f_k(\bft{x}) 
			\leq p_g g_k(\bft{x}) \right\} \\
		S &= \left\{\bft{x} \in \mathbb{R}^{\nvar} : p_k f_k(\bft{x}) 
			\leq \max_{l \neq k} \{ p_l f_l(\bft{x})\} \right\}
	\end{align}
	and denote their respective complements by $T'$ and $S'$. Then,
	\begin{align}
		Q_{e, k} &= \int_T p_k f_k(\bft{x}) \myd \bft{x} + \int_{T'} 
		p_g g_k(\bft{x}) \myd \bft{x} \\
		P_{e, k} &= \int_{S} p_k f_k(\bft{x}) \myd \bft{x} + \int_{S'} 
		p_g g_k(\bft{x}) \myd \bft{x}.
	\end{align}
	Since $p_g g_k(\bft{x}) = \sum_{l \neq k} p_l f_l(\bft{x}) \leq 
	\max_{l \neq k} p_l f_l(\bft{x})$ it holds that $S \subseteq T$ and 
	$T' \subseteq S'$. Hence,
	\begin{align}
		\int_{T} p_k f_k(\bft{x}) \myd \bft{x} &= \int_{S} p_k 
		f_k(\bft{x}) \myd \bft{x} + \int_{T \setminus S} p_k 
		f_k(\bft{x}) \myd \bft{x} \\
		\int_{S'} p_g g_k(\bft{x}) \myd \bft{x} &= \int_{T'} p_g 
		g_k(\bft{x}) \myd \bft{x} + \int_{S' \setminus T'} p_g 
		g_k(\bft{x}) \myd \bft{x}
	\end{align}
	However, the sets $T \setminus S$ and $S' \setminus T'$ both equal
	\begin{equation}
		U = \left\{ \bft{x} \in \mathbb{R}^{\nvar} : \max_{l\neq k} \{ 
			p_l f_l(\bft{x}) \} < p_k f_k(\bft{x}) \leq p_g 
			g_k(\bft{x}) \right\},
	\end{equation}
	so it follows that
	\begin{equation}
		Q_{e, k} = P_{e, k} + \int_{U} p_k f_k(\bft{x}) \myd \bft{x} - 
		\int_{U} p_g g_k(\bft{x}) \myd \bft{x} \leq P_{e, k}
	\end{equation}
	by definition of the set $U$.
\end{proof}

This provides a lower bound for $P_{e, k}$ in terms of $u(f_k, g_k)$. What 
remains to be shown is that $P_{e, k}$ has an upper bound in terms of $u(f_k, 
g_k)$.  No such bound has yet been found. However, the following result can be 
presented which does bound $P_{e, k}$ from above.

\begin{theorem}
	For a single class $k$ the OvR-BER is smaller than or equal to the sum 
	of the pairwise BER estimates involving class $k$.
\end{theorem}

\begin{proof}
	Recall that the OvR-BER for class $k$ is given by
	\begin{equation}
		P_{e, k} = \qquad \qquad \smashoperator{
			\int\limits_{
				\strut
				p_k f_k(\bft{x}) \leq
				\max\limits_{l \neq k} \{ p_l f_l(\bft{x}) \}
			}
		} p_k f_k(\bft{x}) \myd \bft{x}
		\quad + \quad \sum_{c \neq k} \quad
		\smashoperator{
			\int\limits_{
				\strut
				p_k f_k(\bft{x}) >
				\max\limits_{l \neq k} \{ p_l f_l(\bft{x}) \}
			}
		} p_c f_c(\bft{x}) \myd \bft{x},
	\end{equation}
	and denote the sum of the pairwise BERs involving $k$ as $F_{e,k}$ 
	given by,
	\begin{align}
		F_{e, k} &= \sum_{c \neq k} \int \min \{ p_k f_k(\bft{x}), p_c 
		f_c(\bft{x}) \} \myd \bft{x} \\
		&=
		\sum_{c \neq k} \quad \smashoperator{
			\int\limits_{
				\strut
				p_k f_k(\bft{x}) < p_c f_c(\bft{x})
			}
		} p_k f_k(\bft{x}) \myd \bft{x}
		\quad + \quad
		\sum_{c \neq k} \quad \smashoperator{
			\int\limits_{
				\strut
				p_k f_k(\bft{x}) > p_c f_c(\bft{x})
			}
		} p_c f_c(\bft{x}) \myd \bft{x}.
	\end{align}
	Then comparing the first term of $F_{e, k}$ with that of $P_{e, k}$ 
	shows
	\begin{equation}
		\sum_{c \neq k} \quad \smashoperator{
			\int\limits_{
				\strut
				p_k f_k(\bft{x}) < p_c f_c(\bft{x})
			}
		} p_k f_k(\bft{x}) \myd \bft{x}
		\quad \geq \quad
		\smashoperator{
			\int\limits_{
				\strut
				p_k f_k(\bft{x}) \leq
				\max\limits_{l \neq k} \{ p_l f_l(\bft{x}) \}
			}
		} p_k f_k(\bft{x}) \myd \bft{x},
	\end{equation}
	since the area of integration on the left is larger than on the right.  
	Similarly,
	\begin{equation}
		\sum_{c \neq k} \quad \smashoperator{
			\int\limits_{
				\strut
				p_k f_k(\bft{x}) > p_c f_c(\bft{x})
			}
		} p_c f_c(\bft{x}) \myd \bft{x}
		\quad \geq \quad
		\sum_{c \neq k} \quad \smashoperator{
			\int\limits_{
				\strut
				p_k f_k(\bft{x}) >
				\max\limits_{l \neq k} \{ p_l f_l(\bft{x})\}
			}
		} p_c f_c(\bft{x}) \myd \bft{x}
	\end{equation}
	for the same reason. This completes the proof.
\end{proof}

\section{Jensen-Shannon Bound Inequality}
\label{app:jensen_shannon}

In this section a proof is given for the statement that the Henze-Penrose 
upper bound on the Bayes error rate is tighter than the Jensen-Shannon upper 
bound derived by \citet{lin1991divergence}. Before presenting the proof, the 
following lemma is presented.

\begin{lemma}
	\label{lem:ineq}
	For $x, y > 0$ it holds that
	\begin{equation}
		x \log \left(1 + \frac{y}{x}\right) + y\log\left(1 + 
			\frac{x}{y}\right) \geq \frac{4\log(2)x y}{x + y}.
	\end{equation}
\end{lemma}
\begin{proof}
	Let $t = \tfrac{y}{x}$ and multiply both sides by $\tfrac{1}{t} + 1 > 
	0$, then the inequality reduces to
	\begin{equation}
		\left(\frac{1}{t} + 1\right)\log\left( 1 + t \right) + (1 + t)
		\log\left( 1 + \frac{1}{t} \right) \geq 4\log(2).
	\end{equation}
	Denote the left hand side by $f(t)$. The proof will now proceed by 
	showing that $f(t) \geq f(1) = 4 \log(2)$ for all $t > 0$. The 
	derivatives of $f(t)$ are given by
	\begin{equation}
		f'(t) = \log\left(1 + \frac{1}{t}\right) - \frac{\log(1 + 
			t)}{t^2} \qquad\textrm{ and }\qquad f''(t) = 
		\frac{2\log(1 + t) - t}{t^3}.
	\end{equation}
	Write the numerator of $f''(t)$ as $g(t)$ such that
	\begin{equation}
		g(t) = 2\log(1 + t) - t, \qquad\textrm{ and }\qquad g'(t) = 
		\frac{1 - t}{1 + t}.
	\end{equation}
	Then it is clear that $g'(t) > 0$ for $0 < t < 1$ and $g'(t) < 0$ for 
	$t > 1$. Furthermore $\lim_{t \rightarrow 0^+} g(t) = 0$ and $\lim_{t 
		\rightarrow \infty} g(t) = -\infty$. Thus, it follows that 
	$g(t)$ increases on $0 < t < 1$ and decreases for $t > 1$. Let $t = 
	\alpha > 1$ be such that $g(\alpha) = 0$, then $g(t) > 0$ for $0 < t < 
	\alpha$ and $g(t) < 0$ for $t > \alpha$.

	From this it follows that $f''(t) > 0$ for $0 < t < \alpha$ and 
	$f''(t) < 0$ for $t > \alpha$. Hence, $f'(t)$ is increasing on $0 < t 
	< \alpha$ and decreasing for $t > \alpha$. Moreover, $\lim_{t 
		\rightarrow 0^+} f'(t) = -\infty$ and $f'(1) = 0$. Thus, it 
	follows that $f'(t)$ is negative on $0 < t < 1$, positive for $t > 1$, 
	and attains a maximum at $t = \alpha$ after which it decreases to 
	$\lim_{t \rightarrow \infty} f'(t) = 0$. Since $\lim_{t \rightarrow 
		0^+} f(t) = \infty$ it follows that $f(t)$ is decreasing on $0 
	< t < 1$ and increasing for $t > 1$.
\end{proof}

\begin{definition}[Kullback-Leibler Divergence]
	For probability density functions $f_1(\bft{x})$\\
	and $f_2(\bft{x})$ the Kullback-Leibler divergence is given by
	\begin{equation}
		D_{KL}(f_1 \| f_2) = \int f_1(\bft{x}) \log_2 
		\frac{f_1(\bft{x})}{f_2(\bft{x})} \myd \bft{x},
	\end{equation}
	\citet{kullback1951information}.
\end{definition}

\begin{definition}[Jensen-Shannon Divergence]
	According to \citet{el1997agnostic} the Jensen-Shannon divergence for 
	two probability density functions $f_1(\bft{x})$ and $f_2(\bft{x})$ 
	with prior probabilities $p_1$ and $p_2$, can be stated in terms of 
	the Kullback-Leibler divergence as
	\begin{equation}
		JS(f_1, f_2) = p_1 D_{KL}(f_1 \| M) + p_2 D_{KL}(f_2 \| M)
	\end{equation}
	with $M(\bft{x}) = p_1 f_1(\bft{x}) + p_2 f_2(\bft{x})$ the mixture 
	distribution and $p_2 = 1 - p_1$.
\end{definition}

\begin{theorem}
	For $p_1 = p_2 = \tfrac{1}{2}$ the Henze-Penrose upper bound on the 
	BER is tighter than the Jensen-Shannon upper bound of 
	\citet{lin1991divergence},
	\begin{equation}
		P_e(f_1, f_2) \leq \tfrac{1}{2} - \tfrac{1}{2}u_{HP} \leq 
		\tfrac{1}{2}J,
	\end{equation}
	where $J = H(p_1) - JS(f_1, f_2)$ with $H(p_1)$ the binary entropy and 
	$JS(f_1, f_2)$ the Jensen-Shannon divergence.
\end{theorem}

\begin{proof}
	First, note that with $p_1 = p_2 = \tfrac{1}{2}$ the binary entropy 
	$H(p_1) = H(\tfrac{1}{2}) = 1$. Second, for the equiprobable case it 
	holds that
	\begin{equation}
		\tfrac{1}{2} - \tfrac{1}{2}u_{HP}(f_1, f_2) = \int 
		\frac{f_1(\bft{x}) f_2(\bft{x})}{f_1(\bft{x}) + f_2(\bft{x})} 
		\myd \bft{x}.
	\end{equation}
	The Jensen-Shannon upper bound can be written as
	\begin{align}
		\tfrac{1}{2} J &= \tfrac{1}{2} - \tfrac{1}{4} \int 
		f_1(\bft{x}) \log_2 \frac{2 f_1(\bft{x})}{f_1(\bft{x}) + 
			f_2(\bft{x})} \myd \bft{x} - \tfrac{1}{4} \int 
		f_2(\bft{x}) \log_2 \frac{2 f_2(\bft{x})}{f_1(\bft{x}) + 
			f_2(\bft{x})} \myd \bft{x} \\
		&= \tfrac{1}{4} \int f_1(\bft{x}) + f_2(\bft{x}) \myd \bft{x} 
		- \tfrac{1}{4} \int f_1(\bft{x}) \log_2 \frac{2 
			f_1(\bft{x})}{f_1(\bft{x}) + f_2(\bft{x})} \myd 
		\bft{x} \\
		&\qquad -\tfrac{1}{4} \int f_2(\bft{x}) \log_2 \frac{2 
			f_2(\bft{x})}{f_1(\bft{x}) + f_2(\bft{x})} \myd 
		\bft{x} \nonumber \\
		&= \tfrac{1}{4} \int f_1(\bft{x}) \left[ 1 - \log_2 \frac{2 
				f_1(\bft{x})}{f_1(\bft{x}) + f_2(\bft{x})} 
		\right] \myd \bft{x} \\
		&\qquad +\tfrac{1}{4} \int f_2(\bft{x})\left[ 1 - \log_2 
			\frac{2 f_2(\bft{x})}{f_1(\bft{x}) + f_2(\bft{x})} 
		\right] \myd \bft{x} \\
		&= \tfrac{1}{4} \int f_1(\bft{x}) \log_2 \left(1 + 
			\frac{f_2(\bft{x})}{f_1(\bft{x}} \right) + 
		f_2(\bft{x}) \log_2 \left( 1 + 
			\frac{f_1(\bft{x})}{f_2(\bft{x})} \right) \myd \bft{x}
	\end{align}
	By Lemma \ref{lem:ineq} it follows that
	\begin{equation}
		f_1(\bft{x}) \log_2 \left(1 + \frac{f_2(\bft{x})}{f_1(\bft{x}} 
		\right) + f_2(\bft{x}) \log_2 \left( 1 + 
			\frac{f_1(\bft{x})}{f_2(\bft{x})} \right)  \geq 
		\frac{4f_1(\bft{x})f_2(\bft{x})}{f_1(\bft{x}) + f_2(\bft{x})},
	\end{equation}
	and therefore
	\begin{equation}
		\tfrac{1}{2} J \geq \tfrac{1}{4} \int 
		\frac{4f_1(\bft{x})f_2(\bft{x})}{f_1(\bft{x}) + f_2(\bft{x})} 
		\myd \bft{x} = \tfrac{1}{2} - \tfrac{1}{2}u_{HP}(f_1, f_2).
	\end{equation}
\end{proof}

\section{Additional Simulation Results}
\label{app:mich_additional}

In this section some additional simulation results are presented for the 
SmartSVM experiments presented in Section~\ref{sec:smartsvm}.  
Table~\ref{tab:mich_avg_time} shows the average time per hyperparameter 
configuration for each of the methods. This is especially useful for comparing 
GenSVM \citep{van2016gensvm} with the other methods, as it has a larger set of 
hyperparameters to consider.

A commonly used tool to summarize results of simulation experiments is to use 
rank plots \citep{demvsar2006statistical}. For each dataset the methods are 
ranked, with the best method receiving rank 1 and the worst method receiving 
rank $6$ (since there are $6$ methods in this experiment). In case of ties 
fractional ranks are used. By averaging the ranks over all datasets, a visual 
summary of the results can be obtained. Figures~\ref{fig:mich_rank_perf}, 
\ref{fig:mich_rank_time_total} and \ref{fig:mich_rank_time_average} show these 
average ranks for predictive performance, total training time, and average 
training time respectively.

The ordering of OvO and SmartSVM in the rank plots for training time may seem 
counterintuitive, considering that SmartSVM is more often the fastest method.  
This can be explained by the fact that in the cases where SmartSVM is slower 
than OvO it is usually also slower than DAG. In contrast, where SmartSVM is 
the fastest method OvO is usually the second fastest method. Because of this, 
SmartSVM obtains a slightly higher average rank than OvO.

\begin{table}[t]
	\centering
	\caption[Average training time SmartSVM experiments]{Average training 
		time per hyperparameter configuration in seconds, averaged 
		over the 5 nested CV folds. Minimal values per dataset are 
		underlined. \label{tab:mich_avg_time}}
	\begin{tabular}{lrrrrrrr}
		Dataset & C \& S & DAG & GenSVM & OvO & OvR & SmartSVM \\
		\hline
		abalone & 702 & 7.712 & 254 & 7.067 & 11.282 & 
		\underline{3.960} \\
		fars & 8329 & 164 & 601 & \underline{151} & 296 & 184 \\
		flare & 24.606 & 0.369 & 1.467 & 0.344 & 0.431 & 
		\underline{0.263} \\
		krkopt & 4526 & 38.337 & 165 & 35.813 & 73.053 & 
		\underline{35.627} \\
		letter & 1668 & 21.438 & 131 & \underline{20.076} & 94.977 & 
		34.454 \\
		nursery & 119 & 3.932 & 3.985 & 3.617 & 4.994 & 
		\underline{3.133} \\
		optdigits & 4.630 & 1.053 & 92.441 & 0.971 & 1.278 & 
		\underline{0.590} \\
		pageblocks & 27.550 & 1.408 & 5.852 & \underline{1.326} & 
		2.559 & 1.504 \\
		pendigits & 132 & 1.589 & 19.273 & \underline{1.498} & 7.965 & 
		2.872 \\
		satimage & 273 & 3.931 & 7.494 & \underline{3.685} & 9.918 & 
		3.942 \\
		segment & 19.846 & 0.376 & 4.509 & 0.370 & 0.985 & 
		\underline{0.285} \\
		shuttle & 403 & 30.042 & 42.744 & \underline{29.571} & 105 & 
		43.928 \\
		texture & 30.954 & 1.816 & 43.200 & 1.727 & 3.665 & 
		\underline{1.028} \\
		winered & 37.179 & 0.646 & 1.587 & 0.586 & 0.966 & 
		\underline{0.465} \\
		winewhite & 135 & 3.115 & 7.514 & 2.903 & 3.732 & 
		\underline{1.954} \\
		yeast & 63.655 & 0.877 & 3.327 & 0.807 & 1.272 & 
		\underline{0.699} \\
		\hline
	\end{tabular}
\end{table}

\begin{figure}[ht]
	\def\RankFigureWidth{.75\textwidth}
	\centering
	\subfloat[][Predictive performance]{%
		\includegraphics[width=\RankFigureWidth]{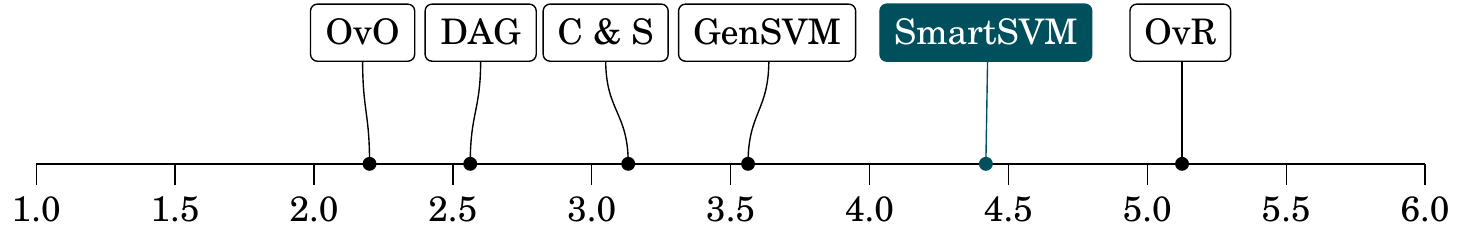}
		\label{fig:mich_rank_perf}
	}\qquad
	\subfloat[][Total training time]{%
		\includegraphics[width=\RankFigureWidth]{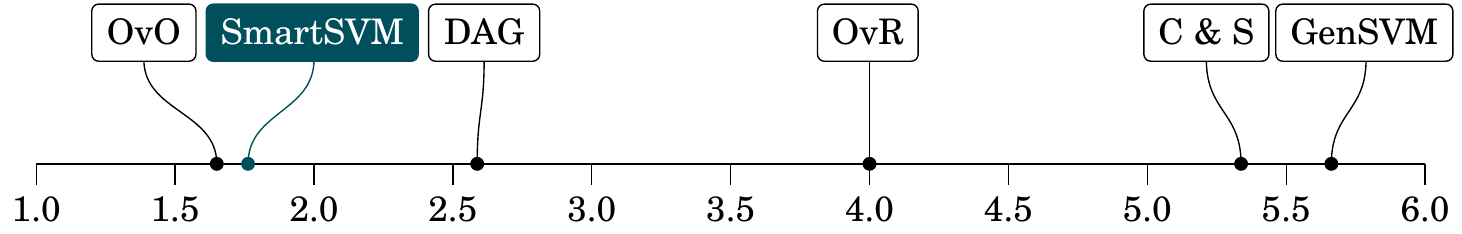}
		\label{fig:mich_rank_time_total}
	}\qquad
	\subfloat[][Average training time]{%
		\includegraphics[width=\RankFigureWidth]{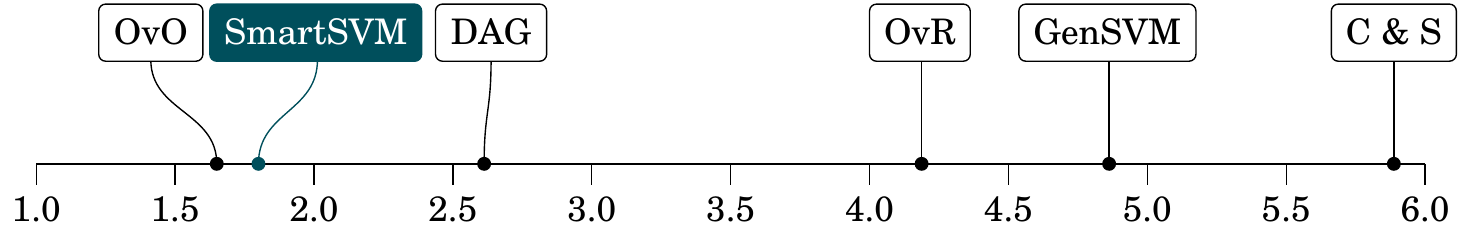}
		\label{fig:mich_rank_time_average}
	}\qquad
	\caption[Rank plots of simulation study]{Rank plots of classifier 
		performance in the simulation study.  
		Figure~\subref{fig:mich_rank_perf} shows the average ranks for 
		out-of-sample predictive performance as measured by the ARI.  
		Figures~\subref{fig:mich_rank_time_total} and 
		\subref{fig:mich_rank_time_average} respectively show the 
		average ranks for total training time and average training 
		time per hyperparameter configuration. 
		\label{fig:mich_rank_plots}}
\end{figure}

\FloatBarrier
\bibliographystyle{bibstyle}
\bibliography{references}

\end{document}